\documentclass{article}

\usepackage{microtype}
\usepackage{graphicx}
\usepackage{subfigure}
\usepackage{booktabs} % for professional tables
\usepackage{amsmath}
\usepackage{amsthm}
\usepackage{algorithmic}
\usepackage{multirow}
\newtheorem{theorem}{Theorem}

\newtheorem{remark}{Remark}[theorem]
\newtheorem{proposition}{Proposition}

\usepackage[accepted]{icml2021} 
\icmltitlerunning{Geometry of Score Based Generative Models}

\begin{document}

\twocolumn[
\icmltitle{Geometry of Score Based Generative Models}



\icmlsetsymbol{equal}{*}

\begin{icmlauthorlist}
\icmlauthor{Sandesh Ghimire}{neu}
% \icmlauthor{Bauiu C.~Yyyy}{equal,to,goo}
\icmlauthor{Jinyang Liu}{neu}
\icmlauthor{Armand Comas}{neu}
\icmlauthor{Davin Hill}{neu}
\icmlauthor{Aria Masoomi}{neu}\\
\icmlauthor{Octavia Camps}{equal,neu}
\icmlauthor{Jennifer Dy}{equal,neu}

\end{icmlauthorlist}

\icmlaffiliation{neu}{Department of Electrical and Computer Engineering, Northeastern University, MA, USA}
% \icmlaffiliation{goo}{Googol ShallowMind, New London, Michigan, USA}
% \icmlaffiliation{ed}{School of Computation, University of Edenborrow, Edenborrow, United Kingdom}

\icmlcorrespondingauthor{Sandesh Ghimire}{drsandeshghimire@gmail.com}
% \icmlcorrespondingauthor{Eee Pppp}{ep@eden.co.uk}

\icmlkeywords{Machine Learning, ICML}

\vskip 0.3in
]



\printAffiliationsAndNotice{\icmlEqualContribution} % otherwise use the standard text.

\begin{abstract}
In this work, we look at Score-based generative models (also called diffusion generative models) from a geometric perspective. From a new view point, we prove that both the forward and backward process of adding noise and generating from noise are Wasserstein gradient flow in the space of probability measures. We are the first to prove this connection. Our understanding of Score-based (and Diffusion) generative models have matured and become more complete by drawing ideas from different fields like Bayesian inference, control theory, stochastic differential equation and Schrodinger bridge. However, many open questions and challenges remain. One problem, for example, is how to decrease the sampling time? We demonstrate that looking from geometric perspective enables us to answer many of these questions and provide new interpretations to some known results. Furthermore, geometric perspective enables us to devise an intuitive geometric solution to the problem of faster sampling. By augmenting traditional score-based generative models with a projection step, we show that we can generate high quality images with significantly fewer sampling-steps.
\end{abstract}

\section{Introduction}
\begin{figure}[t]
% \vskip 0.2in
\begin{center}
\centerline{\includegraphics[width=\columnwidth]{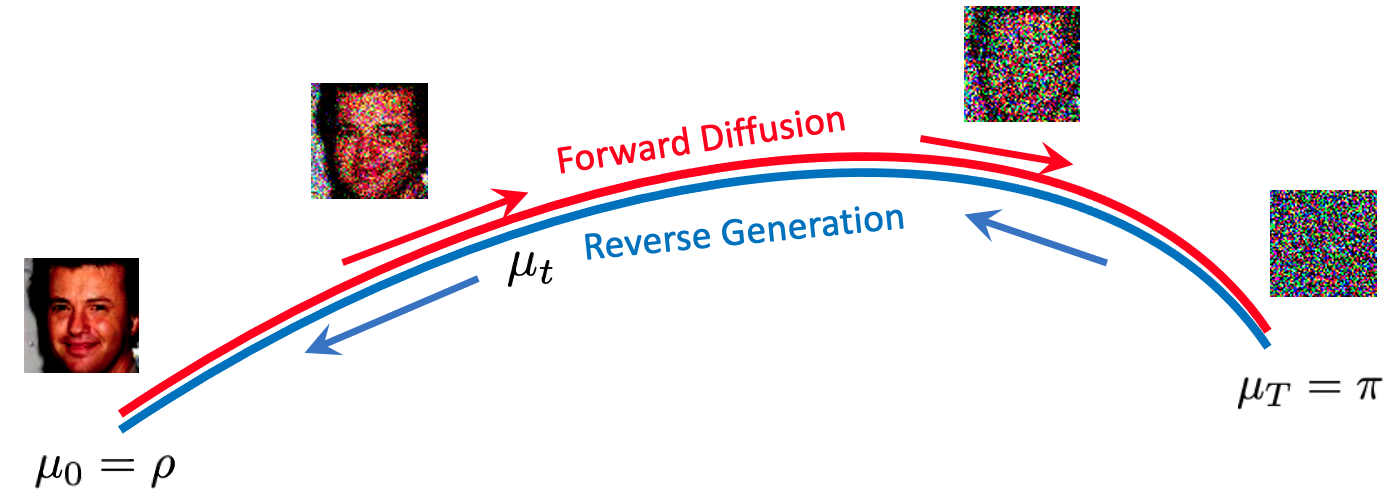}}
\caption{Both the forward diffusion process and reverse generation process in diffusion generative models correspond to Wasserstein gradient flow along the same gradient-flow-path.}
\label{wasserstein_grad}
\end{center}
% \vskip -0.2i
\vspace{-0.8cm}
\end{figure}
Score-based (or Diffusion) models are a new type of generative models in the field of computer vision and machine learning, achieving state-of-the-art results in image synthesis \cite{diffusion_beats_gan} and log likelihood \cite{diffusion_variational}. They have recently gained popularity due to interesting applications such as text to image generation (DALL-E \cite{dalle2}, \cite{latent_diffusion} and Imagen \cite{imagen2022photorealistic}), image super-resolution, image editing \cite{sdedit_meng}, etc. Score-based generative models have enjoyed diverse perspectives from different fields. Originally, diffusion models (DDPM) were developed from expected lower bound (ELBO) maximization on data log likelihood  \cite{Ho_ddpm}. \citet{song_datagradient} showed that we can learn gradient of log likelihood (called score functions) and use it to generate images. \citet{song2021scorebased} showed that the epsilon function in DDPM is in fact the scaled version of the score function. They further generalized these models to a continuous setting as stochastic differential equations \cite{song2021scorebased}. More recent works have connected score-based generative models with Schrodinger bridge problem \cite{Bortoli2021diffusion} and control theoretic perspectives \cite{likelihood_chen2022, likelihood_huang2021}. 

In this work, we present a completely different view-point on score-based generative models: geometric perspective.
To the best of our knowledge, we are the first to explore the geometric connection of these generative models. Applying the solid mathematical framework in the area of Wasserstein gradient flow \cite{jko, ambrosio2005gradient, wibisono2018sampling, salim2020wasserstein, korba2020non}, we show that the forward and backward process of adding noise and generating image from the noise are in fact equivalent to moving on a gradient-flow-path in a metric space of probability distributions following the Wasserstein gradient flow equation. 

While our understanding of score-based generative models has matured over time, few important questions remain unanswered. For example, why is it a good idea to choose forward and reverse variance the same? Can we choose reverse variance differently? Are score-based generative models same as the energy based models \cite{gen_convnet, flow_contrastive_ebm, comet}?  Furthermore, new models have been proposed like Wavefit \cite{wavefit}, which tries to generalize the diffusion sampling to a proximal gradient type of update. How can we explain this type of algorithms? In this work, we demonstrate that the geometric connection investigated in this work helps to answer these questions from a geometric point of view. 

In addition to conceptual advantages and novel perspectives, geometric framework enables us to design practical algorithms with faster sampling capability. Score-based generative models work remarkably well when the number of sampling-steps is large (i.e. small step-size). However, the sampling time is also large for such finer schemes. As we decrease the number of sampling steps, the samples move away from the gradient-flow-path incurring error in each step, and resulting into high overall error. To minimize such error and achieve high quality samples even with small number of sampling steps, we propose to project back intermediate samples to the gradient-flow-path after every step. To achieve this, we propose an efficient estimation of Wasserstein gradient to descend towards the flow-path. As demonstrated in the result section, our proposed method significantly reduces error for smaller number of sampling steps.
Below we summarize our contributions. All complete proofs are included in the appendix.
% \begin{enumerate}
%     \item First we prove that the forward diffusion is accelerated Wasserstein Gradient Flow 
%     \item We show that the score-based generative model and diffusion model both are reverse of the forward gradient flow 
%     \item We show that the JKO can be used to accelerate the process and JKO generalizes the score-based algorithm
%     \item We show that by using a convex functional with data distribution, we can accelerate the samplign process
%     \item We propose a new algorithm based on these insights and show that we improve significantly on sampling time without compromising the quality.
% \end{enumerate}

\begin{enumerate}
    \item To the best of our knowledge, this is the first work to theoretically prove the connection between score-based generative models and Wasserstein gradient flow. We establish this relationship through Theorems \ref{forward_thm} and \ref{reverse_flow}.
    
    \item This connection sheds light on several interesting questions: 1) the reverse variance in score-based generative models, 2) the connection between score-based model and energy based model, and 3) the use of proximal gradient algorithms as proposed in recent works.
    
    \item Based on these insights, we propose a new algorithm which generalizes the score-based model and allows for significantly faster sampling, which would otherwise be very difficult to achieve. To achieve this, we also propose an efficient Wasserstein gradient estimation algorithm.
    
\end{enumerate}

\section{Related Works}
Early works on diffusion models were based on matching the forward and reverse joint distributions through bounds on log likelihood \cite{Ho_ddpm, sohl2015deep}. \cite{song_datagradient} proposed a score-based generative model motivating from the Langevin dynamics and estimating the score function. Later \cite{song2021scorebased} showed that the two approaches are actually equivalent and it can be generalized further in continuous time setting through the stochastic differential equations. On the more theoretical directions, score-based optimization has been shown to be equivalent to likelihood maximization through Feynman-Kac theorem \cite{likelihood_chen2022, likelihood_huang2021}. Other notable works are interpreting the forward diffusion and generation as solving the Schrodinger Bridge problem \cite{Bortoli2021diffusion}. Many approaches have been proposed to speed up the sampling process through clever ways to solve differential equations \cite{fast_solver}. 

In their seminal work, Jordan, Kinderlehrer, Otto (JKO) proved the connection between Wasserstein gradient flow and the the diffusion systems guided by Fokker-Planck equations \cite{jko}. This result has been vastly generalized and formalized by \cite{villani2003topics, villani2009optimal} and \cite{ambrosio2005gradient} giving birth to the theory of Wasserstein gradient flow and optimization on space of probability measures. Several notable works have followed in machine learning \cite{wibisono2018sampling}, \cite{korba2020non}, \cite{salim2020wasserstein}, for example for sampling, generative models, etc.

\section{Preliminaries}
\subsection{Notations}
Let $\mathcal{B}(\mathcal{X})$ denote the Borel $\sigma$- algebra over $\mathcal{X}$, and let $\mu$ denote a probability measure on $\mathcal{X}$. $\mathcal{P}_2(\mathcal{X})$ denotes the space of probability measures $\mu$ on $\mathcal{X}$ with finite second order moment.
For any $\mu \in \mathcal{P}_2(\mathcal{X})$, $L^2(\mu)$ is the space of functions $f: \mathcal{X} \to \mathcal{X}$ such that $\int ||f||^2 d\mu < \infty$ \cite{ambrosio2005gradient, korba2020non}. Let $T: \mathcal{X} \to \mathcal{X} $, then $T_{\#} \mu$ denotes the pushforward measure of $\mu$ by $T$ such that the transfer lemma $\int \phi(T(x))d\mu(x) = \int \phi(y)dT_{}\mu(y)$ holds for any measurable bounded function $\phi$. We use Wasserstein-2 distance as a metric on the space of probability measures. The Wasserstein-2 distance is defined as $W_2^2 (\mu, \nu) = \inf_{s \in S(\mu, \nu)} \int ||x-y||^2 ds(x,y)$, where $\mu, \nu \in P_2(\mathcal{X})$ and $\mathcal{S}(\mu, \nu)$ is the set of couplings between $\mu$ and $\nu$, i.e. the set of nonnegative measures, $s$ over $\mathcal{X} \times \mathcal{X}$ such that their projections on first and second components are $P_{\#}s = \mu$ and $Q_{\#}s = \nu$ where $P: (x,y) \mapsto x$ and $Q: (x,y) \mapsto y$ \cite{villani2003topics}.

\subsection{Wasserstein Gradient Flow}
Let $(\mu_t)_{t \in (0,T)}$ denote a family of probability measures.
This family satisfies a continuity equation if there exists a family of velocity fields, $(v_t)_{t \in (0,T)}$ such that 
\begin{align}
\label{continuity_eq}
    \frac{\partial \mu_t}{\partial t} + div(\mu_t v_t) = 0
\end{align}
 in a distributional sense. It is also absolutely continuous if $||v_t||_{L^2(\mu_t)}$ is integrable over $(0,T)$. Among all possible $v_t$, there is one with minimum $L^2(\mu_t)$ norm, and it lies on the tangent space of $\mathcal{P}_2(\mathcal{X})$ and is called tangent vector field \cite{ambrosio2005gradient} Chapter 8.

We define a functional on the space of probability measures, $\mathcal{F}(\mu) : \mathcal{P}_2(\mathcal{X}) \to (-\infty, \infty)$. We define Wasserstein gradient of any functionals on the $\mathcal{P}_2(\mathcal{X})$ space as the change in the value of functional with small perturbation on the probability measure. Wasserstein gradient can be expressed in the following form \cite{ambrosio2005gradient} Chapter 10:
\begin{align}
    \nabla_{W_2} \mathcal{F} = \nabla \mathcal{F'}(\mu)
\end{align}
Consider the KL divergence, $\text{KL}(\mu || \pi)$ between any measure $\mu$ and a base measure $\pi$. We can show that the Wasserstein gradient of the functional $\text{KL}(.|| \pi)$ at $\mu$ is 
\begin{align}
    \nabla_{W_2} \text{KL}(.|| \pi) = \nabla \log (\frac{ \mu}{\pi})
\end{align}
In the family $(\mu_t)_{t \in (0,T)}$, let the initial measure be $\mu_0 = \rho$ and final measure is $\mu_T = \pi$. Then there exists a geodesic between two probability measures $\rho$ and $\pi$ with respect to the Wasserstein metric. If we choose the velocity field equal to the negative of Wasserstein gradient (i.e. $v_t = - \nabla_{W_2} \text{KL}(.|| \pi) $), then we can show that the path traced by the probability measures is the geodesic between $\rho$ and $\pi$ \cite{ambrosio2005gradient} Chapter 7, and the flow is known as Wasserstein gradient flow. Using the functional $\text{KL}(.|| \pi)$ and the continuity equation, we obtain the equation of Wasserstein gradient flow as:
\begin{align}
    \frac{\partial \mu_t}{\partial t} = div \big[\mu_t \nabla \log (\frac{ \mu_t}{\pi})]
\end{align}
Wasserstein gradient flow is a differential equation of probability measures. Consider a Wasserstein gradient flow with initial measure $\mu_0$ satisfying the continuity equation \ref{continuity_eq}. Let $x_0 \sim \mu_0$ be sample from the initial measure. The differential equation for the samples can be derived from continuity equation as follows \cite{ambrosio2005gradient}:
\begin{align}
\label{mu_to_x}
    \dot{x} = v_t
\end{align}
\subsection{Score Based Generative Model}
%Write Fokker Plank and SDE.
%Write Fokker Plank with our notation
Score-based generative model \cite{song2021scorebased} extends diffusion models to work on continuous time setting using stochastic differential equations (SDEs). The forward and reverse process of adding noise and generating images are interpreted as forward and reverse diffusion process with following differential equations: 
\begin{align}
\label{for_sde}
    \text{FOR}:dx &= f(x,t)dt +g_t dw\\
    \nonumber
    & t: 0\to T, x_0 \sim \rho\\
\label{rev_sde}
    \text{REV}:dx &= [f(x,t) - g_t^2 \nabla_x \log \mu_t(x)] dt + g_t d\bar{w}\\
    \nonumber
    & t: T\to 0, x_T \sim \pi
\end{align}
where $f$ is forward drift function and $dw$ is the Brownian motion. Note that the flow of time in two SDE is different: the time flows from $0$ to $T$ in the forward process and the initial distribution is $\rho$, while the time flows from $T$ to $0$ in the reverse process. Time direction is crucial in the stochastic differential equation because for the forward process, $x_t$ is independent of the future $t' >t$ while in the reverse direction it is independent of the past \cite{anderson1982reverse}. To make things simpler such that time always flow in positive direction, we can equivalently use the positive time notation indexed by $\tau$ (following is equivalent to eq.(\ref{rev_sde})):
\begin{align}
\label{rev_sde_tau}
    dx &= [- f(x,\tau) + g_{\tau}^2 \nabla_x \log \mu_{\tau}(x)] d\tau + g_{\tau} d\bar{w}\\
    \nonumber
    & \tau: 0\to T, x_0 \sim \pi, \tau = T - t
\end{align}
Note that we use $t$ for forward flow of time and $\tau = T-t$ for the backward flow of time, so that $\tau$ now flows from $0$ to $T$. With this notation, $\mu_{t=0} = \mu_{\tau = T} = \rho$, $\mu_{t=T} = \mu_{\tau = 0} = \pi$, and $\beta_t = \beta_{T-\tau}$. Euler-Maruyama discretization of the reverse SDE equation yields:
\begin{align}
\label{discrete_reverse_sde}
    x_{\tau+\delta \tau} = x_{\tau}  +(g_{\tau}^2 \nabla_x \log \mu_{\tau}(x) - f(x,\tau)) \delta \tau + g_{\tau} z 
\end{align}
where $z \sim N(0, I)$ is a random normal distributed sample.
\section{Forward Diffusion as Gradient Flow}
Instead of taking the velocity vector to be negative of Wasserstein gradient, we consider an accelerated flow where at any time, $t$, the velocity is equal to the negative Wasserstein gradient scaled by a time-varying $\beta_t$.
\begin{proposition}[Accelerated Wasserstein Gradient Flow]
We define accelerated gradient flow with respect to the functional $\mathcal{F}$ as the gradient flow where the velocity vector is defined as $v_t = -\beta_t \nabla_{W_2} \mathcal{F}$. Consequently, the continuity equation is given by:
\begin{align}
\label{accelerated_gf}
    \frac{\partial \mu_t}{\partial t} = div \big[\mu_t \beta_t \nabla_{W_2} \mathcal{F} ]
\end{align}
\end{proposition}
Using this accelerated Wasserstein Gradient flow, we can establish a connection with the forward process in score-based generative model. We start from the Fokker-Planck equation corresponding to the stochastic differential equation of the forward diffusion process given by eq.(\ref{for_sde}):
\begin{align}
\label{fokker_planck}
    \frac{\partial \mu_t}{\partial t} = -div (\mu_t f) + \frac{1}{2}div(\nabla(g_t^2 \mu_t))
\end{align}
% We can rewrite this equation as:
% \begin{align}
%     \frac{\partial \mu_t}{\partial t} &= -div (\mu_t f) + div(\mu_t\frac{1}{2}g_t^2\nabla( \log \mu_t))\\
%     &= -div \big(\mu_t (f - \frac{1}{2}g_t^2\nabla \log \mu_t )\big)
% \end{align}
where the initial measure is $\mu_0 = \rho$. Following this SDE, we know that it will end up in the final measure, $\mu_T = \pi$.
Next theorem shows that the forward Fokker Planck equation and the accelerated Wasserstein Gradient descent are equivalent.
\begin{theorem}
\label{forward_thm}
Consider an accelerated gradient flow in eq.(\ref{accelerated_gf}) with initial measure $\mu_0 = \rho$ and the target measure $\mu_T = \pi$ and the functional on the Wasserstein space defined by $\mathcal{F}(\mu) = \text{KL}(.||\pi)$. The family of measures corresponding to this gradient flow is equivalent to the family of measures corresponding to the forward Fokker Plank equation in eq.(\ref{fokker_planck}) given that $f$ and $\beta_t$ take the following form:  $f = {\beta_t} \nabla \log \pi$, $\beta_t = \frac{g_t^2}{2}$.
\end{theorem}
\begin{remark}
Consider the special case with measure $\mu_T = \pi =\mathcal{N}(0,I) = {\exp(-\frac{||x||^2}{2})}/{Z}$, we get $f = -\beta x$ and the forward diffusion equation is given by the following SDE:
\begin{align}
    dx = -{\beta_t}x dt +\sqrt{2\beta_t}dw
\end{align}
which is exactly the forward flow of DDPM model \cite{Ho_ddpm, song2021scorebased}.
\end{remark}

This implies that the forward diffusion process considered in the diffusion generative model, DDPM \cite{Ho_ddpm, song2021scorebased} can be equivalently thought as an accelerated Wasserstein gradient flow starting from an initial measure $\mu_0 = \rho$ corresponding to the data distribution and following the negative gradient towards the target measure $\mu_T = \mathcal{N}(0, I)$.
\begin{remark}
We can also think of accelerated Wasserstein gradient flow as regular Wasserstein gradient flow with non-uniform discretization, \textit{i.e.,} step at $t$ is scaled by $\beta_t$. 
\end{remark}

Next we investigate the geometric interpretation of the generation process or the reverse SDE.

\section{Generation as Reverse Gradient Flow}
 Next theorem establishes the equivalence between the reverse SDE and the Wasserstein Gradient flow.  
 
\begin{theorem}
\label{reverse_flow}
The reverse SDE in eq.(\ref{rev_sde_tau}) is equivalent to the Wasserstein gradient flow in the space of probability measures with respect to the functional $\mathcal{F}(\mu ) = - \text{KL}(.||\pi)$ starting from the initial measure $\mu_{\tau = 0} = \pi$ towards the target measure $\mu_{\tau = T} = \rho$.
\end{theorem}
\begin{proof}
    \begin{align}
    \mathcal{F}(\mu) &= -KL(\mu||\pi)\\
    & = \int -\log \mu d\mu + \int \log \pi d\mu \\
    \label{add_subtract}
    & = \color{red}{\underbrace{\color{black}{\int (-2 \log \mu  + \log \pi)     
    d\mu}}_{{\mathcal{G}}}} \color{black}{+} \color{blue} {\underbrace{\color{black}{ \int \log \mu d\mu}}_{\mathcal{H}}}
\end{align}
Here, we apply forward backward splitting scheme due to \cite{wibisono2018sampling, salim2020wasserstein}
\begin{align}
\label{descend}
 \nu_{\tau} &= (I - \beta_{\tau} \nabla_{{W}_2}\mathcal{G}(\mu_{\tau}))_{\#} \mu_{\tau}\\
 \label{jko_step}
\mu_{\tau+ \delta \tau} &= JKO_{\beta_{\tau} \mathcal{H}}(\nu_{\tau})
\end{align}
where $\nabla_{{W}_2}\mathcal{G}(\mu)$ is the Wasserstein gradient, the expression for which can be obtained as:
\begin{align}
    \nabla_{W_2} \mathcal{G} = \nabla \mathcal{G'}(\mu) = \nabla(-2 \log \mu  + \log \pi)
\end{align}
 In eq.(\ref{descend}), we are trying to move in the direction of Wasserstein gradient. Let samples $x_{\tau}$ from distribution $x_{\tau} \sim \mu_{\tau}$. Transforming differential equation in measure space to sample space, similar to eq.(\ref{mu_to_x}), yields:
 \begin{align}
    &y_{\tau} = x_{\tau} -\beta_{\tau} \nabla(-2 \log \mu_t(x_{\tau}) + \log \pi(x_{\tau}) ) \delta \tau
    % &x_{\tau + \delta \tau}  = y_{\tau} + \sqrt{2\beta_{\tau}} z_{\tau}
\end{align}
In eq.(\ref{jko_step}), we are using JKO operator as a solution of the negative entropy functional, $\mathcal{H}$, where the JKO operator is defined as : $$JKO_{\beta, \mathcal{H}}(\nu) =  \underset{\zeta \in \mathcal{P}_2(\mathcal{X})}{\text{argmin}} \mathcal{H}(\zeta) + \frac{1}{2 \beta} W_2^2(\zeta, \nu)$$ 

For the negative entropy functional, we have the exact solution as Brownian motion \cite{jko, wibisono2018sampling, salim2020wasserstein}. Let $y_{\tau} \sim \nu_{\tau}$, we obtain
\begin{align}
    &y_{\tau} = x_{\tau} -\beta_{\tau} \nabla(-2 \log \mu_t(x_{\tau}) + \log \pi(x_{\tau})) \delta \tau\\
    &x_{\tau + \delta \tau}  = y_{\tau} + \sqrt{2\beta_{\tau}} z_{\tau}
\end{align}
Combining both, we obtain
\begin{align}
\nonumber
   x_{\tau+\delta \tau} &= x_{\tau} +(2\beta_{\tau} \nabla \log \mu_{\tau}(x_{\tau}) - \beta_{\tau}\nabla \log \pi(x_{\tau}))\delta \tau \\
   &+ \sqrt{2\beta_{\tau}} z_{\tau}
\end{align}
In the limiting case as $\delta \tau \to 0$, we obtain,
\begin{align}
\nonumber
   dx =  (2\beta_{\tau} \nabla \log \mu_{\tau}(x_{\tau}) - \beta_{\tau}\nabla \log \pi(x_{\tau}))d \tau 
   + \sqrt{2\beta_{\tau}} dw
\end{align}
which coincides exactly with the reverse SDE in eq.(\ref{rev_sde_tau}) for $g_{\tau}^2 = 2\beta_{\tau}$ and $f = \beta_{\tau}\nabla \log \pi$.
\end{proof}
 The reverse SDE or the score-based model is trying to reverse the forward process by tracing the path followed in the forward process in the opposite direction. One important implication of this theorem is that since we are moving towards the target measure $\mu_T = \mathcal{N}(0, I)$ in the forward process, the reverse is actually simply moving away from $\mathcal{N}(0, I)$, which is realized as the accelerated Wasserstein gradient flow with the functional $- \text{KL}(.||\pi)$. The gradient flow path with constant velocity is the geodesic. Since we are considering gradient flow path with acceleration, it is not exactly the geodesic, but similar path traced by gradient flow. We will call it gradient-flow-path in rest of the paper.
\section{Insights, Connections, Discussion}
We have shown that both the forward and reverse diffusion process involved in score-based generative models are gradient flows on the space of probability measures. 
% Theorem 2 implies that the generation process of score-based generative model is simply moving tangentially at every $\tau$ away from the final measure $\pi$. 
We gain geometric insights because of this geometric interpretation.

\subsection{Alternative Interpretation of Reverse SDE equation}
Score-based generative model uses the fact that for every forward SDE of the form in eq.(\ref{for_sde}), there exists a reverse SDE as in eq.(\ref{rev_sde_tau}), which is a remarkable result due to \cite{anderson1982reverse}. Theorem \ref{reverse_flow} provides an interesting interpretation of this result from a completely different perspective. In eq.(\ref{add_subtract}), we added and subtracted the negative entropy term $\int \log \mu d\mu$ in diffusion and drift terms respectively. It allowed us to design a forward-backward algorithm instead of forward algorithm of Wasserstein gradient flow. The backward term essentially added the Brownian motion term yielding us a reverse stochastic differential equation. Note that if we had not added and subtracted the term $\int \log \mu d\mu$, we would have obtained following iterative scheme:
\begin{align}
    x_{\tau+\delta \tau} &= x_{\tau} +(\beta_{\tau} \nabla \log \mu_{\tau}(x_{\tau}) - \beta_{\tau}\nabla \log \pi(x_{\tau}))\delta \tau
\end{align}
Note that this is a discretized version of the following ODE. 
\begin{align}
\label{rev_ode}
   \frac{dx}{d\tau} = -f(x, \tau) + \frac{1}{2}g_{\tau}^2 \nabla \log \mu_{\tau}(x), \hspace{0.2cm} \tau: [0 \to T]
\end{align}
Comparing this equation with eq.(\ref{rev_sde_tau}), observe that eq.(\ref{rev_sde_tau}) has an additional $\frac{1}{2}g_t^2 \nabla \log \mu_t(x)$ in the drift part which is compensated by the Brownian motion $g_{\tau} d\Tilde{w}$. It is clear to see that eq.(\ref{rev_sde_tau}) and eq.(\ref{rev_ode}) yields same family of marginal distributions at $\tau \in [0,T]$ even though the former is deterministic differential equation and the latter is the stochastic. Perhaps the advantage of score-based models is that stochasticity helps in generating diverse samples for small sample size.

\subsection{Why is reverse variance same as the forward variance?}
In the DDPM model \cite{Ho_ddpm}, it was not clear how to choose the variance of the reverse differential equation, and why choosing the reverse variance the same as in forward is a good strategy. From previous analysis, we see that the reverse variance must be same as the forward because we have added and subtracted the same negative entropy term from both the drift and the diffusion. However, it is possible to change the reverse time variance. For example, we can add $\alpha  \int \log \mu d\mu $ to both the drift and diffusion terms in eq. (\ref{add_subtract}). Then the reverse SDE variance will be $\sqrt{2 \alpha \beta_t }$, but then the drift term in eq.(\ref{rev_sde_tau}) will also be modified to $\frac{1+\alpha}{2}g_t^2 \nabla \log \mu_t(x)$ instead of $g_t^2 \nabla \log \mu_t(x)$.

\subsection{Contrasting Score-based with Energy-based Model}
Let's assume that the probability measure of data can be obtained in the form of $\rho \propto \exp(-V)$. Consider Wasserstein gradient descent with the functional as $KL(.||\rho)$
\begin{align}
    \mathcal{F}(\mu) &= KL(\mu||\rho)\\
    & =  \int V d\mu + \int \log \mu d\mu
\end{align}
We can use the same forward-backward splitting scheme as we used in Theorem 2 Proof, and with similar reasoning, we can recover the Langevin dynamics:
\begin{align}
    x_{\tau+\delta \tau} &= x_{\tau}  - \beta_{\tau}\nabla V(x_{\tau})\delta \tau 
   + \sqrt{2\beta_{\tau}} z
\end{align}
This demonstrates the critical difference between the Energy-based model and the score-based model: while the energy based model is moving towards the data distribution $\mu_0 = \rho$ with functional $KL(.||\rho)$, the score-based model is moving away from the isotropic Gaussian distribution ($\pi$) with the functional, $-KL(.||\pi)$ . Score-based generative model traces the forward diffusion path in the reverse direction thereby avoiding the need to work with the data distribution $\rho$. In the energy based model, however, we need to either estimate energy function like $V$ \cite{flow_contrastive_ebm, comet} or KL divergence with the data, $\rho$. 

\subsection{Proximal Algorithms in Diffusion models}
WaveFit \cite{wavefit} tries to generalize the iteration in diffusion models to a proximal algorithm. Motivating from a fixed point iteration, they try to improve upon DDPM model by drawing ideas from GANs and propose a proximal algorithm type of approach which is faster in generating samples than DDPM without losing quality. Here, we show that starting from geometric perspective we can reach the proximal algorithm as a way to perform Wasserstein gradient descent. Consider a functional, say $\mathcal{F}(\mu) = KL(\mu||\rho)$ for example where $\rho$ is the data distribution. Forward discretization of Wasserstein gradient descent yields us iteration \cite{jko, salim2020wasserstein}
\begin{align}
    \mu_{\tau'} 
    = \underset{\nu \in \mathcal{P}_2(\mathcal{X})}{\text{argmin}} \mathcal{F}(\nu) + \frac{1}{2 \gamma} W_2^2(\nu, \mu)
\end{align}

which is a proximal gradient algorithm in the space of probability measures. This justifies why proximal algorithms make sense in the context of diffusion generative models or score-based generative models because we are trying to reach the data distribution descending in the direction of Wasserstein gradient. \citet{jko, wibisono2018sampling, salim2020wasserstein} have shown that proximal algorithm converges to the target distribution, $\rho$. As for the choice of functional $\mathcal{F}$, it can be any convex functional that decreases as we descend towards the target measure $\rho$. Wavefit \cite{wavefit} shows that using much stronger GAN-type objective as a functional yields good result.

% \subsection{GANs vs Diffusion models}
% Our geometric interpretation sheds light on connection between different generative models like diffusion models, energy based models and GANs.

% Similarly, we can think of GANs as our model with large step size so that in one step we reach from $\mu_T$ to $\mu_0$

\section{Challenges with Faster Sampling}
\begin{figure}[t!]
% \vskip 0.2in
\begin{center}
\centerline{\includegraphics[width=\columnwidth]{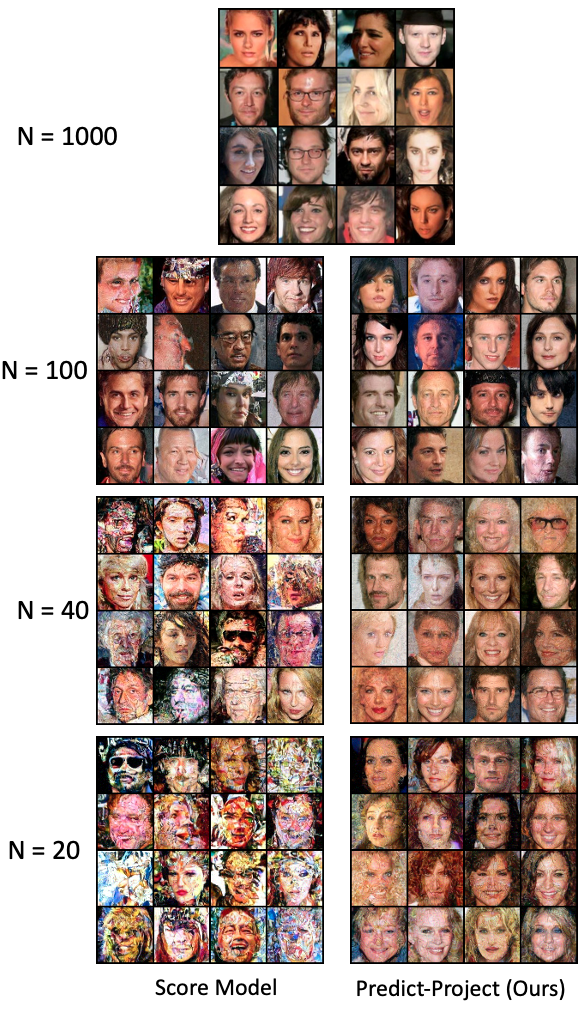}}
\caption{Qualitative comparison of Celeb-A samples generated by Score-based model and our model at different number of sampling steps (N). Our model maintains reasonable quality even when N decreases down to 20. }
\label{fig:celeba}
\end{center}
\vspace{-0.7cm}
\end{figure}
Once the connection between the gradient flow and score-based generative model is established, we can interpret the generation as a process walking on the gradient-flow-path. If we follow the flow-path with small steps, we can reliably reach the initial data distribution, as demonstrated by the success of the score-based generative models and diffusion models. However, this is not a great idea if we increase the step size. The score-based increment is a linear approximation and therefore it accrues more error as we increase the step size.  It has been experimentally observed that the samples get poorer as we increase the step size in score-based models \cite{diffusion_beats_gan, Ho_ddpm, song2021scorebased}. 
From a geometric point-of-view, we are taking Wasserstein gradient steps using forward-backward strategy. While this strategy works well when the step size is small, it converges to a biased measure for large step-size. Bias associated with the forward-backward strategy for large step size has been studied in the context of Wasserstein gradient flow \cite{wibisono2018sampling}. In our case, this issue is further exacerbated by the fact that the functional $-\text{KL}(\mu|| \pi)$ we are trying to minimize is actually concave with respect to $\mu$.

\begin{figure}[t]
% \vskip 0.2in
\begin{center}
\centerline{\includegraphics[width=\columnwidth]{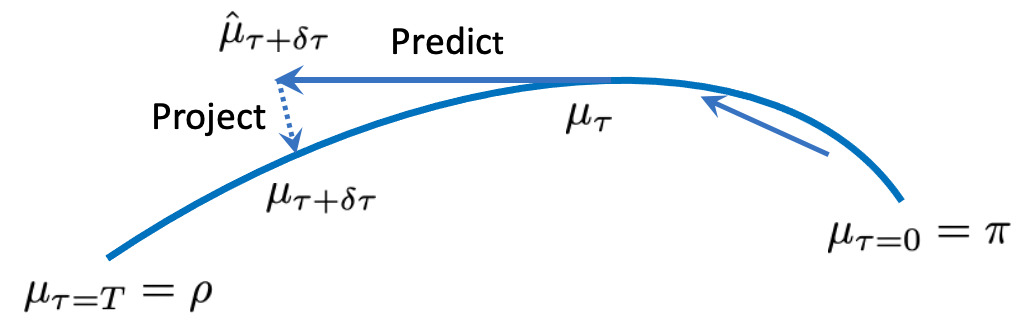}}
\caption{During generation, sample moves tangentially to the gradient-flow-path in the reverse direction. If the step-size is large, it incurs error, which we mitigate by projection to the gradient-flow-path.}
\label{fig:predict_project}
\end{center}
\vspace{-0.5cm}
\end{figure}

To mitigate this issue, we propose an intuitive and geometric idea: projection. As shown in Fig.\ref{fig:predict_project}, as we try to sample in score-based models with large step-size, the error gets large and the trajectory deviates away from the gradient-flow-path. We propose to resolve this problem by projecting again to the gradient-flow-path before taking another step.

\section{Projection to Gradient-flow-path}
\begin{algorithm}[tb]
   \caption{Predict-Project Sampling}
   \label{alg:sampling}
\begin{algorithmic}[1]
   % \STATE {\bfseries Input:} data $x_i$, size $m$
   % \REPEAT
   \STATE Inputs: $N$, $\delta \mu$, $s_{\theta}^{*}$, $T_{\theta}^{*}$
   \STATE $x \sim \mathcal{N}(0,I)$, 
   $\delta \tau = 1/N$
   \FOR{$\tau=0$ {\bfseries to} $1 - 1/N$}
   \STATE $x^{pred} = x +(2\beta_{\tau} s_{\theta}^{*}(x) - \beta_{\tau}\nabla \log \pi(x_{\tau}))\delta \tau$
   \STATE $x^{proj} = x^{pred} + T_{\theta}^*(x^{pred}, \tau + \delta \tau) .\sqrt{\beta_\tau}.{\delta \mu}. \delta \tau $
   \STATE $z \sim \mathcal{N}(0,I)$
   \STATE $x = x^{proj} + \sqrt{2\beta_t}z$
   % \IF{$x_i > x_{i+1}$}
   % \STATE Swap $x_i$ and $x_{i+1}$
   % \STATE $noChange = false$
   % \ENDIF
   \ENDFOR
   \STATE RETURN $x$ 
   % \UNTIL{$noChange$ is $true$}
\end{algorithmic}
\end{algorithm}
Score-based generative model first trains a score model, $s$ such that $s_{\theta}^*(x_\tau, \tau) = \nabla \log \mu_{\tau}(x_{\tau})$ using score matching strategy. Once, the score model is trained, the discretized Eurler-Maruyama step (eq.(\ref{discrete_reverse_sde})) is used for generation of samples, where score function, $s^*$ replaces $\nabla \log \mu_{\tau}(x_{\tau})$:
\begin{align}
\nonumber
    x_{\tau+\delta \tau} &= x_{\tau} +(2\beta_{\tau} s_{\theta}^*(x_\tau, \tau) - \beta_{\tau}\nabla \log \pi(x_{\tau}))\delta \tau \\
    \nonumber
   &+ \sqrt{2\beta_{\tau}} z\\
   \label{eq:pred}
   & = x_{\tau+\delta \tau}^{pred} + \sqrt{2\beta_{\tau}} z_{\tau}
\end{align}
 It can also be interpreted as predict and diffuse steps, where predict step is $x_{\tau+\delta \tau}^{pred} = x_{\tau} +(2\beta_{\tau} s_{\theta}^*(x_\tau, \tau) - \beta_{\tau}\nabla \log \pi(x_{\tau}))\delta \tau$. 
Since generation is trying to trace the gradient-flow-path, after each of these predict-diffuse, we should obtain samples from the measure $\mu_{\tau + \delta \tau}$ on the gradient-flow-path. Because of discretization error and bias, these samples do not lie on the gradient-flow-path, in fact they deviate away. To pull these samples towards the measure, $\mu_{\tau + \delta \tau}$ on the gradient-flow-path, we use the fact that we have a way to sample from the measures on gradient-flow-path.  Using the SDE equation, we can write closed form conditional distribution of $\mu_{\tau + \delta \tau}$ as follows:
\begin{align}
\nonumber
    \mu_{\tau+\delta \tau | T}(x_{\tau + \delta \tau} | x_{ T}) = \mathcal{N}(x_{\tau+\delta \tau};x_{\tau+\delta \tau}^{mean}(x_{\tau=T}), 2 \beta_{\tau+ \delta \tau}I)
\end{align}
Sampling from this conditional distribution is given by the following equation:
\begin{align}
\label{eq:sample_geodesic}
    x_{\tau+\delta \tau} = x_{\tau+\delta \tau}^{mean}(x_{\tau = T}) + \sqrt{2\beta_{\tau + \delta \tau}} z_{\tau}
\end{align}
Comparing eq.(\ref{eq:sample_geodesic}) with eq.(\ref{eq:pred}), we note that pulling $x_{\tau+\delta \tau}^{pred}$ close to $x_{\tau+\delta \tau}^{mean}$ may be enough to pull the samples $x_{\tau+\delta \tau} $ in eq.(\ref{eq:pred}) towards gradient-flow-path assuming that $\beta_{\tau + \delta \tau}$ is close to $\beta_{\tau }$. In terms of measures, we consider the measure associated with samples $x_{\tau+\delta \tau}^{pred}$ and target samples $x_{\tau+\delta \tau}^{mean}$. Let's define the measure corresponding to samples $x^{pred}_{\tau+\delta \tau}$ as $\mu^{pred}_{\tau+\delta \tau}$ and the measures corresponding to means,  $x_{\tau+\delta \tau}^{mean}$ as $\mu^{mean}_{\tau+\delta \tau}$. We can sample from the measure $\mu^{mean}_{\tau+\delta \tau}$ by first sampling from $x_{\tau=T} \sim \mu_{\tau=T}$, and passing them through $x_{\tau+\delta \tau}^{mean}$. Our strategy to project the samples in eq.(\ref{eq:pred}) onto gradient-flow-path is to project the pred measure, $\mu^{pred}_{\tau+\delta \tau}$ to the mean measure $\mu^{mean}_{\tau+\delta \tau}$. We achieve this through Wasserstein gradient descent in the space of probability measures. For that, we need an efficient way to estimate Wasserstein gradient, which we describe in next subsection.

\subsection{ Efficient Estimation of Wasserstein Gradient }

Imagine that we want to estimate Wasserstein Gradient of a functional $\mathcal{J}(\mu)$, \textit{i.e.}, $\nabla_{W_2}\mathcal{J}(\mu)$. For that, we can use the following Taylor expansion:
\begin{align}
\label{Wgradient_est}
\mathcal{J}((I +h T)_{\# \mu} ) = \mathcal{J}(\mu) + h \langle \nabla_{W_2}\mathcal{J}(\mu), T \rangle_{\mu} + o(h)
\end{align}
where, $\nabla_{W_2}\mathcal{J}(\mu) \in L^2(\mu)$ is the Wasserstein gradient of $\mathcal{J}$ at $\mu$. 
To estimate the Wasserstein gradient, consider the following optimization problem:
\begin{align}
   T^* =  \underset T {argmin} \hspace{0.2cm}{\langle \nabla_{W_2}\mathcal{J}(\mu), T \rangle_{\mu} + \frac{1}{2h}||T||^2 _{\mu}}
\end{align}
It is easy to see that $T^*  = -h \nabla_{W_2}\mathcal{J}(\mu)$ is the solution of this problem. Plug in eq.(\ref{Wgradient_est}), and parameterize the gradient function $T$ as a function of neural network parameters $\theta$. Hence, we solve the following optimization problem:
\begin{align}
\label{efficient_Wasserstein_grad}
   \underset {\theta} {min} \hspace{0.2cm}{\mathcal{J}((I + T_\theta)_{\# \mu} ) + \frac{1}{2h}||T_{\theta}||^2 _{\mu}}
\end{align}
 This optimization is efficient and can use parallel processing because: 1) it only requires samples from the measure $\mu$, and 2) we can use minibatch from the measure $\mu$ to update neural network parameters at a time. This removes the need to obtain all samples at a time leading to stochastic gradient descent optimization of $\theta$. 

\subsection{Predict-Project Algorithm}
\begin{table}[t]
\caption{Comparison of FID score as a measure of generated image quality between score-based and our generative model at different number of sampling steps N.}
\label{fid_is_score}
\begin{center}
\begin{small}
% \begin{sc}
\resizebox{\columnwidth}{!}{%
\begin{tabular}{lcccccr}
\toprule
{Datasets}
&& N = 1000 & N = 100 & N = 40 & N = 20 \\ 
\hline
\multicolumn{6}{c}{FID Score $\downarrow$}\\
\hline
\multirow{ 2}{*}{Celeb-A}& Score Model   & \multirow{ 2}{*}{6.331}& 35.14 & 149.42 & 222.71 \\
&Predict-Project & & \textbf{20.54} & \textbf{68.23} & \textbf{121.12}
\\ \hline

\multirow{ 2}{*}{LSUN}& Score Model   & \multirow{ 2}{*}{15.12}& 34.62 & 122.23 & 246.17 \\
&Predict-Project & & \textbf{25.35} & \textbf{66.61} & \textbf{164.32}
\\ \hline
\multirow{ 2}{*}{SVHN}& Score Model   & \multirow{ 2}{*}{18.95}&  \textbf{146.63} & 183.30 & 285.61 \\
&Predict-Project & & 149.56 & \textbf{ 174.34} & \textbf{152.94 }\\
\bottomrule
\end{tabular}
}
% \end{sc}
\end{small}
\end{center}
\vspace{-0.4cm}
\end{table}

\begin{figure*}[t!]
% \vskip 0.2in
\begin{center}
\centerline{\includegraphics[width=0.9\textwidth]{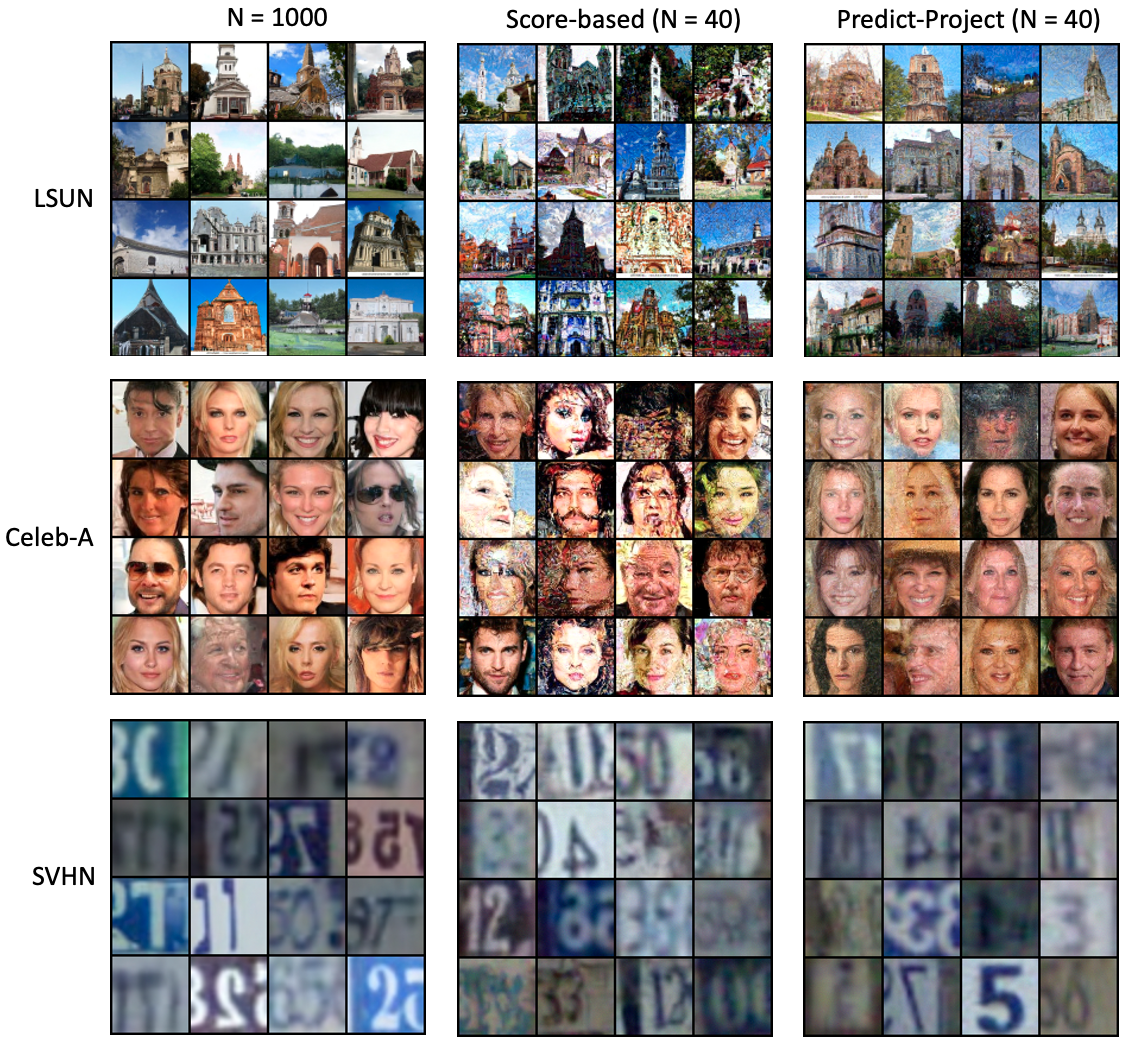}}
\caption{Comparison of generated samples in three datasets when number of sampling steps is decreased to as low as N = 40.}
\label{fig:N_40}
\end{center}
\vspace{-0.5cm}
\end{figure*}

With the Wasserstein gradient estimation method in hand, we now move on to project $\mu_{\tau+\delta \tau}^{pred}$ to $\mu_{\tau+\delta \tau}^{mean}$. We define the functional $\mathcal{J}$ in the following way:
\begin{align}
\nonumber
    &\mathcal{J}_{\tau+\delta \tau}((I + T_{\theta, \tau+\delta \tau})_{\#}{ \mu_{\tau+\delta \tau}^{pred}} )\\
    &= \int ||x_{\tau+\delta \tau}^{pred} + T_{\theta}(x_{\tau+\delta \tau}^{pred}) - x_{\tau+\delta \tau}^{mean}||^2 d\mu_{\tau+\delta \tau}^{pred}(x_{\tau+\delta \tau}^{pred})
\end{align}
Note that, $T_{\theta, \tau+\delta \tau}$ is indexed by time. Instead of learning different $T_{\theta}$ for different time, we parameterize it by time as $T(., \tau)$ as in score function \cite{song2021scorebased}. Similarly, we choose $h = 1/\sqrt{\beta_{\tau}}$ to be different for different $\tau$ in eq.(\ref{efficient_Wasserstein_grad}). To train the projection function, $T$, we sample $\tau $ from the uniform distribution in the interval $(0,1]$, $x_{\tau=T} \sim \mu_{\tau=T}$ and optimize the following optimization:
\begin{align}
\nonumber
     \underset {\theta} {min} \hspace{0.2cm} E_{\tau, x_{\tau = T}}& \big [ ||x_{\tau+\delta \tau}^{pred} + T_{\theta}(x_{\tau +\delta \tau}^{pred}, \tau +\delta \tau ) - x_{\tau +\delta \tau}^{mean}||^2  \\
     \nonumber
     & + \frac{\sqrt{\beta_{\tau}}}{2}||T_{\theta}(x_{\tau +\delta \tau}^{pred}, \tau +\delta \tau)||^2 \big ]
\end{align}
After training, we have, 
$\sqrt{\beta_{\tau}}T_{\theta}^*  = - \nabla_{W_2}\mathcal{J}(\mu^{pred})$. Using this relation, we update the sample as 
\begin{align}
    x^{proj} &= x^{pred} - \nabla_{W_2}\mathcal{J}(\mu^{pred})(x^{pred}).{\delta \mu}. \delta \tau\\
    &= x^{pred} +\sqrt{\beta_\tau} T_{\theta}^*(x^{pred}, \tau) .{\delta \mu}. \delta \tau
\end{align}
where $\delta \mu$ is the small scalar by which to move in the direction of Wasserstein gradient and $\delta \tau$ is present due to the fact that the Wasserstein gradient flow with velocity field $v_t$ corresponds to dynamics $\dot{x_{\tau}} = v_{\tau}(x_{\tau})$ (see eq.(\ref{mu_to_x})). See Algorithm \ref{alg:sampling} for full sampling algorithm.

\section{Experimental Results}
To demonstrate efficacy of our algorithm, we train and generate samples on three datasets: 1) Celeb-A dataset, 2) LSUN-church dataset, and 3) SVHN dataset, where all images are of size $64 \times 64$. Our neural network architecture for both score model and projection model uses standard U-net architecture with attention due to \cite{diffusion_beats_gan}. We use publicly available code from \cite{song2021scorebased} as score-based model. In these experiments, we demonstrate that as we decrease the number of sampling steps, the quality of samples decreases in Score-based generative model, but we maintain quality to a reasonable level even when the number of sampling steps is reduced to as low as 20. We use FID metric \cite{fid_gans} to measure the sample qualities. 

In Table \ref{fid_is_score}, we compare the FID score of generated images from score-based method and our Predict-Project method. We outperform the score-based method by a large margin in all cases except SVHN (N=100). This underperformance could be because our model is not trained well in SVHN (see appendix) due to lack of time. For qualitative comparison, please see Fig. (\ref{fig:N_40}) and Fig. (\ref{fig:celeba}). These results our claim that projecting to the gradient-flow-path improves sample quality, especially when the number of sampling-step is low. 

\section{Conclusion}

We presented a novel geometric perspective on score-based generative models (also called diffusion generative models) by showing that they are in fact gradient flows in a space of probability measures. The geometric insight gained from this connection helped us answer and clarify some critical open questions. We also demonstrated that it can help us design faster sampling algorithm. We believe that this connection will help diffuse knowledge between Wasserstein gradient flow field and score-based generative models field in the future inspiring interesting solutions to problems in both areas. Similarly, connection with energy-based models, proximal algorithms and reverse SDE could help design better algorithms in general and generative models in specific. Energy-based models, for example, can be combined with score-based models in the light of geometric understanding.

\newpage

\newpage

\bibliography{mainbib}
\bibliographystyle{icml2021}

\appendix
\newpage
\onecolumn
\section{Proof of Theorems}
\subsection{Forward Diffusion as Gradient Flow}
\begin{theorem}
Consider an accelerated gradient flow in eq.(\ref{accelerated_gf}) with initial measure $\mu_0 = \rho$ and the target measure $\mu_T = \pi$ and the functional on the Wasserstein space defined by $\mathcal{F}(\mu) = \text{KL}(.||\pi)$. The family of measures corresponding to this gradient flow is equivalent to the family of measures corresponding to the forward Fokker Plank equation in eq.(\ref{fokker_planck}) given that $f$ and $\beta_t$ take the following form:  $f = {\beta_t} \nabla \log \pi$, $\beta_t = \frac{g_t^2}{2}$.
\end{theorem}

\begin{proof}
The Fokker Planck equation in eq.(\ref{fokker_planck}) is 
\begin{align}
    \frac{\partial \mu_t}{\partial t} = -div (\mu_t f) + \frac{1}{2}div(\nabla(g_t^2 \mu_t))
\end{align}
We can rewrite this equation as:
\begin{align}
    \frac{\partial \mu_t}{\partial t} &= -div (\mu_t f) + div(\mu_t\frac{1}{2}g_t^2\nabla( \log \mu_t))\\
    \label{fokker_planck_app}
    &= -div \big(\mu_t (f - \frac{1}{2}g_t^2\nabla \log \mu_t )\big)
\end{align}
On the other hand, accelerated Wasserstein gradient flow is given by eq.(\ref{accelerated_gf}): 
\begin{align}
    \frac{\partial \mu_t}{\partial t} &= div \big[\mu_t .\beta_t \nabla \log (\frac{ \mu_t}{\pi})]\\
    \label{accelerated_gf_app}
    & = -div(\mu_t.(\beta_t \nabla \log \pi - \beta_t \nabla \log \mu_t))
\end{align}

Comparing eq.(\ref{fokker_planck_app}) and eq.(\ref{accelerated_gf_app}), it is clear that accelerated Wasserstein gradient Flow is same as the forward stochastic differential equation if we choose:
\begin{align}
    \beta_t &= \frac{1}{2}g_t^2, \text{and}\\
    f &= \beta_t \nabla \log \pi
\end{align}
\end{proof}

\begin{remark}
Consider the special case with measure $\mu_T = \pi =\mathcal{N}(0,I) = {\exp(-\frac{||x||^2}{2})}/{Z}$, we get $f = -\beta x$ and the forward diffusion equation is given by the following SDE:
\begin{align}
    dx = -{\beta_t}x dt +\sqrt{2\beta_t}dw
\end{align}
which is exactly the forward flow of DDPM model \cite{Ho_ddpm, song2021scorebased}.
\end{remark}

\subsection{Generation as Reverse Gradient Flow}
\begin{theorem}
The reverse SDE in eq.(\ref{rev_sde_tau}) is equivalent to the Wasserstein gradient flow in the space of probability measures with respect to the functional $\mathcal{F}(\mu ) = - \text{KL}(.||\pi)$ starting from the initial measure $\mu_{\tau = 0} = \pi$ towards the target measure $\mu_{\tau = T} = \rho$.
\end{theorem}
%%%%%%%%%%%%%%%%%%%%%%%%%%%%%%%%%%%%%%%%%%%%%%%%%%%%%%%%%%%%%%%%%%%%%%%%%%%%%%%
%%%%%%%%%%%%%%%%%%%%%%%%%%%%%%%%%%%%%%%%%%%%%%%%%%%%%%%%%%%%%%%%%%%%%%%%%%%%%%%
\begin{proof}
\begin{align}
    \mathcal{F}(\mu) &= -KL(\mu||\pi)\\
    & = \int -\log \mu d\mu + \int \log \pi d\mu \\
    \label{add_subtract_app}
    & = \color{red}{\underbrace{\color{black}{\int (-2 \log \mu  + \log \pi)     
    d\mu}}_{{\mathcal{G}}}} \color{black}{+} \color{blue} {\underbrace{\color{black}{ \int \log \mu d\mu}}_{\mathcal{H}}}
\end{align}
Here, we apply forward backward splitting scheme due to \cite{wibisono2018sampling, salim2020wasserstein}
\begin{align}
\label{descend_app}
 \nu_{\tau} &= (I - \beta_{\tau} \nabla_{{W}_2}\mathcal{G}(\mu_{\tau}))_{\#} \mu_{\tau}\\
 \label{jko_step_app}
\mu_{\tau+ \delta \tau} &= JKO_{\beta_{\tau} \mathcal{H}}(\nu_{\tau})
\end{align}
where $\nabla_{{W}_2}\mathcal{G}(\mu)$ is the Wasserstein gradient. We can compute the expression for the Wasserstein gradient using the following relation:
\begin{align}
    \nabla_{W_2} \mathcal{G} = \nabla \mathcal{G'}(\mu)
\end{align}
First, the derivative is given by 
\begin{align}
    \mathcal{G}'(\mu) = -2 \log \mu + \log \pi
\end{align}
Therefore,

\begin{align}
    \nabla_{{W}_2}\mathcal{G}(\mu) = \nabla \mathcal{G}'(\mu) = \nabla (-2 \log \mu + \log \pi)
\end{align}
In eq.(\ref{descend_app}), we are trying to move in the direction of Wasserstein gradient. Let samples $x_{\tau}$ from distribution $x_{\tau} \sim \mu_{\tau}$. Transforming differential equation in measure space to sample space, similar to eq.(\ref{mu_to_x}), yields:
 \begin{align}
    &y_{\tau} = x_{\tau} -\beta_{\tau} \nabla(-2 \log \mu_t(x_{\tau}) + \log \pi(x_{\tau}) ) \delta \tau
    % &x_{\tau + \delta \tau}  = y_{\tau} + \sqrt{2\beta_{\tau}} z_{\tau}
\end{align}

In eq.(\ref{jko_step_app}), we are using JKO operator as a solution of the negative entropy functional, $\mathcal{H}$, where the JKO operator is defined as : $$JKO_{\beta, \mathcal{H}}(\nu) =  \underset{\zeta \in \mathcal{P}_2(\mathcal{X})}{\text{argmin}} \mathcal{H}(\zeta) + \frac{1}{2 \beta} W_2^2(\zeta, \nu)$$ 

For the negative entropy functional, we have the exact solution as the Brownian motion \cite{jko, wibisono2018sampling, salim2020wasserstein}. Let $y_{\tau} \sim \nu_{\tau}$, we obtain
\begin{align}
    &y_{\tau} = x_{\tau} -\beta_{\tau} \nabla(-2 \log \mu_t(x_{\tau}) + \log \pi(x_{\tau})) \delta \tau\\
    &x_{\tau + \delta \tau}  = y_{\tau} + \sqrt{2\beta_{\tau}} z_{\tau}
\end{align}
Combining both, we obtain
\begin{align}
\nonumber
   x_{\tau+\delta \tau} &= x_{\tau} +(2\beta_{\tau} \nabla \log \mu_{\tau}(x_{\tau}) - \beta_{\tau}\nabla \log \pi(x_{\tau}))\delta \tau \\
   &+ \sqrt{2\beta_{\tau}} z_{\tau}
\end{align}
In the limiting case as $\delta \tau \to 0$, we obtain,
\begin{align}
\nonumber
   dx =  (2\beta_{\tau} \nabla \log \mu_{\tau}(x_{\tau}) - \beta_{\tau}\nabla \log \pi(x_{\tau}))d \tau 
   + \sqrt{2\beta_{\tau}} dw
\end{align}
which coincides exactly with the reverse SDE in eq.(\ref{rev_sde_tau}) for $g_{\tau}^2 = 2\beta_{\tau}$ and $f = \beta_{\tau}\nabla \log \pi$.
\end{proof}

\section{Experimental Details}
We jointly train the score model $s_{\theta}$ and projection model $T_{\theta}$. They have the same U-Net with attention architecture following \cite{diffusion_beats_gan}. We apply minibatch optimization to optimize both score model and projection model. Because of this the computational burden is low. In terms of parameters, since we have additional, projection model, the parameter is twice of regular score model.

We train Celeb-A model upto 450K iteration and LSUN upto 300K iteration with the batch size of $32$, and report the FID score. We had to terminate SVHN early at 50K iteration (batch size = 32) due to lack of time. We will continue to train this model and will update the score later if we get chance.

\subsection{Hyperparameter $\delta \mu$}
While projecting measures to the gradient-flow-path, we scale the Wasserstein gradient towards path by a factor, $\delta \mu$. This factor intuitively represents how much error is likely to be present in the prediction step of the score model. Obviously, the error is large for $N=20$ and small for $N=100$, so we choose $\delta \mu  \in [0,1]$ large for $N=20$ and small for $N=100$. At the moment, it is a hyperparameter, which we optimize keeping in mind that it should correspond to the level of error is score model prediction. In the future work, we will estimate this hyperparameter from the training loss $||T_{\theta}||_{\mu}^2$. Current best hyperparameter for $\delta \mu$ are:

\begin{align*}
&N = 20,  \hspace{0.2cm}\delta \mu = 1\\
&N= 40,  \hspace{0.2cm} \delta \mu = 0.4\\
&N = 100 ,\hspace{0.2cm} \delta \mu = 0.1\\
\end{align*}

We are cleaning up the code and will make it publicly available.

\end{document}